\documentclass[10pt,twocolumn,letterpaper]{article}

\usepackage{cvpr}
\usepackage{times}
\usepackage{epsfig}
\usepackage{graphicx,xfrac}
\graphicspath{{figures/}{Figures/}}
\usepackage{amsmath,wrapfig, graphicx,bm}
\usepackage{amssymb,amsthm}
\usepackage{paralist}

\DeclareMathOperator*{\argmin}{arg\,min}
\newtheorem{proposition}{Proposition}

\usepackage{multirow}
\usepackage{colortbl, array}
\usepackage{hhline}
\usepackage{xcolor}

\usepackage[utf8]{inputenc} 
\usepackage[T1]{fontenc}    
\usepackage{amsmath,amsthm,amssymb,comment,tabularx}       
\usepackage{url}            
\usepackage{booktabs}       
\usepackage{amsfonts}       
\usepackage{nicefrac}       
\usepackage{microtype}      
\usepackage[linesnumbered,ruled]{algorithm2e}
\usepackage{graphicx,caption,paralist,color,wrapfig}
\usepackage{accents}
\newtheorem{definition}{Definition}

\usepackage{algpseudocode}

\newcommand{\myrowcolour}{\rowcolor[gray]{0.925}}
\newcommand{\highest}[1]{\textcolor{blue}{\mathbf{#1}}}

\definecolor{rulecolor}{RGB}{70,10,171}
\definecolor{tableheadcolor}{RGB}{120,50,200}

\newcommand{\topline}{ %
        \arrayrulecolor{rulecolor}\specialrule{0.1em}{\abovetopsep}{0pt}}%
\newcommand{\midtopline}{ %
        \arrayrulecolor{rulecolor}\specialrule{\lightrulewidth}{0pt}{0pt}}%
\newcommand{\bottomline}{ %
        \arrayrulecolor{rulecolor} \specialrule{\lightrulewidth}{0pt}{0pt}}%

\usepackage[pagebackref=true,breaklinks=true,letterpaper=true,colorlinks,bookmarks=false]{hyperref}

\definecolor{mycol1}{rgb}{0.5961,    0.7255,    0.8157}
\definecolor{mycol2}{RGB}{219,206,186}
\definecolor{mycol3}{rgb}{0.7176    0.8275    0.4667}

\definecolor{mycol4}{RGB}{242    239    249}
\definecolor{mycol5}{RGB}{220    243    242}
\definecolor{mycol6}{RGB}{222    242    217}
\definecolor{mycol7}{RGB}{247    238    237}
\definecolor{myred}{rgb}{1, 0.92, .56}

\cvprfinalcopy 

\def\cvprPaperID{1} 

\ifcvprfinal\fi
\begin{document}

\title{ManifoldNorm: Extending normalizations on Riemannian Manifolds}

\author{Rudrasis Chakraborty\\
University of California\\
Berkeley, CA, USA\\
{\tt\small rudrasischa@gmail.com}
}

\maketitle


\begin{abstract}
  Many measurements in computer vision and machine learning manifest as non-Euclidean data samples. Several researchers recently extended a number of deep neural network architectures for manifold valued data samples. Researchers have proposed models for manifold valued spatial data which are common in medical image processing including processing of diffusion tensor imaging (DTI) where images are fields of $3\times 3$ symmetric positive definite matrices or representation in terms of orientation distribution field (ODF) where the identification is in terms of field on hypersphere. There are other sequential models for manifold valued data that recently researchers have shown to be effective for group difference analysis in study for neuro-degenerative diseases. Although, several of these methods are effective to deal with manifold valued data, the bottleneck includes the instability in optimization for deeper networks. In order to deal with these instabilities, researchers have proposed residual connections for manifold valued data. One of the other remedies to deal with the instabilities including gradient explosion is to use normalization techniques including {\it batch norm} and {\it group norm} etc.. But, so far there is no normalization techniques applicable for manifold valued data. In this work, we propose a general normalization techniques for manifold valued data. We show that our proposed manifold normalization technique have special cases including popular batch norm and group norm techniques. On the experimental side, we focus on two types of manifold valued data including manifold of symmetric positive definite matrices and hypersphere. We show the performance gain in one synthetic experiment for moving MNIST dataset and one real brain image dataset where the representation is in terms of orientation distribution field (ODF). 
\end{abstract}

\section{Introduction}\label{intro}
Geometric deep learning is a relatively nascent field which involves developing techniques to deal with manifold-valued samples, for example, a 2D matrix-valued image where at each pixel we have a matrix. Though traditional deep learning is an obvious choice for processing, in order to process structured matrices one needs to resort to sophisticated geometric tools. Recently, several researchers \cite{bronstein2017geometric,chakraborty2018statistical,cohen2019general,cohen2018spherical,cohen2019gauge,esteves2018learning,chakraborty2018cnn,kondor2018clebsch,sommer2019horizontal,chakraborty2018manifoldnet,kondor2018generalization} proposed deep learning tools tailored for non-Euclidean data. There are two types of data domains applicable for manifold valued deep learning: \begin{inparaenum}[\bfseries (1)] \item each sample is a function on a manifold, i.e., $X_i:\mathcal{M} \rightarrow \mathbf{R}$ \item each sample is manifold valued grid, i.e., $X_i:\mathbf{Z}^n \rightarrow \mathcal{M}$ \end{inparaenum}. A special case for the second type of data domain is grayscale images where $n=2$ and $\mathcal{M} = \mathbf{R}$. 

Some of the recent works where the data domain is function on manifold include Spherical CNN \cite{cohen2018spherical,esteves2018learning,kondor2018clebsch}, Homogeneous CNN \cite{chakraborty2018cnn,kondor2018generalization,cohen2019general}. Cohen et al. \cite{cohen2018spherical} extended the convolution operator on hypersphere and showed that the proposed convolution operator is equivariant to the group of rotations. In Esteves et al. \cite{esteves2018learning}, the authors proposed a different way to do spherical convolution by using the definition proposed by Driscoll and Healy \cite{driscoll1994computing}. Their proposed convolution is equivariant to planar rotations. In \cite{chakraborty2018cnn,cohen2019general,kondor2018generalization}, the authors proposed a more general definition of convolution on a Riemannian homogeneous space and proved that their definition is equivariant to the group that naturally acts on the homogeneous space. Moreover, in \cite{cohen2019gauge}, the authors went one step further and proposed a Gauge equivariant convolution operator. 

Several researchers focused on the second type of data domain where each sample is a manifold valued grid. In \cite{chakraborty2018manifoldnet}, the authors proposed a convolution neural network on a general Riemannian manifold. They proposed a definition of convolution equivariant to the isometry group acts on the underlying manifold. The authors proposed convolution, non-linearity and invariant fully connected layers. In this wok, we propose normalization layer appropriate for the formalism of CNN for a Riemannian manifold proposed in \cite{chakraborty2018manifoldnet}. In \cite{bouza2020mvc}, the authors proposed a CNN for manifold valued data based on defining convolution on tangent spaces. Several other researchers including \cite{chakraborty2018statistical}  proposed a statistical recurrent model for manifold valued sequential datasets.

One of the obstacles in defining deep neural network with a large number of layers is the explosion of gradient. Several ``remedies'' have been proposed including residual connection \cite{he2016deep}, batch normalization \cite{ioffe2015batch}. Recently, authors in \cite{zhen2019dilated} proposed residual connections for convolutions on manifold valued data and have achieved more stable optimization technique. This motivates us to define normalization techniques on a general Riemannian manifold. In \cite{brooks2019riemannian}, the authors proposed batch normalization for manifold of symmetric positive definite matrices. In this work, we generalize the work in two ways \begin{inparaenum}[\bfseries (a)] \item we extend normalization technique for a Riemannian manifold \item moreover, inspired by the recent work of group normalization \cite{wu2018group}, we define group normalization for a general Riemannian manifold. \end{inparaenum}

In this work, our contribution is as follows: \begin{inparaenum}[\bfseries (a)] \item we propose a Riemannian group normalization technique appropriate for Riemannian homoegenous spaces \item we prove for matrix Lie groups our proposed group normalization satisfies the desired first and second order moments \item proof of concept type experiments show the performance gain of our proposed method over several state-of-the-art manifold valued baselines. \end{inparaenum}

\section{Preliminaries}
\label{sec:prelim} 

This section is intended for a very brief summarization of some differential geometric terminologies we are going to use in the rest of the paper. For a more concrete treatment, the readers are encouraged to look at \cite{boothby1986introduction}. 

\begin{definition}[Riemannian manifold and metric]
 Let $(\mathcal{M},g^\mathcal{M})$ be a orientable complete Riemannian manifold with a Riemannian metric $g$, i.e., $\forall x \in \mathcal{M}: g_x:T_x{\mathcal{M}}\times T_x{\mathcal{M}} \rightarrow \mathbf{R}$ is a bi-linear symmetric positive definite map, where $T_x\mathcal{M}$ is the tangent space of $\mathcal{M}$ at $x\in \mathcal{M}$. Let $d: \mathcal{M} \times \mathcal{M} \rightarrow [0,\infty)$ be the distance induced from the Riemannian metric $g$.
\end{definition}

\begin{definition}
\label{theory:def1}
Let $p \in \mathcal{M}$, $r > 0$. Define $\mathcal{B}_r(p) =
\left\{ q \in \mathcal{M} | d(p,q) < r \right\}$ to be a open ball at
$p$ of radius $r$.
\end{definition}

\begin{definition}[Local injectivity radius \cite{groisser2004newton}]
\label{theory:def2}
The local injectivity radius is defined as $r_{\text{inj}}(p) = \sup \left\{ r
| \text{Exp}_p : (\mathcal{B}_r(\mathbf{0}) \subset T_p \mathcal{M} )
\rightarrow \mathcal{M} \text{ is defined}\right.$ $\left.\text{and is a
diffeomorphism}\right.$ $\left.\text{onto its image} \right\}$ at $p \in \mathcal{M}$.
The {\it injectivity radius} \cite{manton2004globally} of
$\mathcal{M}$ is defined as $r_{\text{inj}}(\mathcal{M}) =
\inf_{p \in \mathcal{M}} \left\{r_{\text{inj}}(p)\right\}$.
\end{definition}
Within $\mathcal{B}_r(p)$, where $r \leq r_{\text{inj}}(\mathcal{M})$,
the mapping $ \text{Exp}^{-1}_p: \mathcal{B}_r(p) \rightarrow
\mathcal{U} \subset T_p \mathcal{M} \subset \mathbf{R}^m$, is called the inverse
Exponential/Log map, $m$ is the dimension of $\mathcal{M}$. 

\begin{definition}
\label{theory:def3.5}
Given $p, q
\in \mathcal{U} \subset \mathcal{B}_r(p)$, where $r \leq r_{\text{inj}}(\mathcal{M})$, the (shortest) geodesic is the smooth curve $\Gamma: [0,1] \rightarrow \mathcal{M}$ with $\Gamma(0) = p$, $\Gamma(1) = q$ and $d(p, q) = \int_{[0,1]} \sqrt{g_{\Gamma(t)}\left(\frac{d\Gamma}{dt}, \frac{d\Gamma}{dt}\right)} dt$.

there exists a unique length minimizing geodesic
segment between $p$ and $q$ and the geodesic segment lies entirely in
$\mathcal{U}$.
\end{definition}

\begin{definition}\cite{chavel1984eigenvalues}
\label{theory:def4}
$\mathcal{U} \subset \mathcal{M}$ is strongly convex if for all $p, q
\in \mathcal{U}$, there exists a unique length minimizing geodesic
segment between $p$ and $q$ and the geodesic segment lies entirely in
$\mathcal{U}$.
\end{definition}

\begin{definition} \cite{groisser2004newton}
\label{theory:def5}
Let $p \in \mathcal{M}$. The {\it local convexity radius} at $p$,
$r_{\text{cvx}}(p)$, is defined as $r_{\text{cvx}}(p) = \sup\left\{ r
\leq r_{\text{inj}}(p) | \mathcal{B}_r(p) \text{ is strongly
  convex}\right\}$. The {\it convexity radius} of $\mathcal{M}$ is
defined as $r_{\text{cvx}}(\mathcal{M}) = \inf_{p \in \mathcal{M}}
\left\{ r_{\text{cvx}}(p)\right\}$.
\end{definition}

In rest of the paper, we assume data points are within the geodesic ball of radius less than $\min\{r_{\text{inj}}(\mathcal{M}), r_{\text{cvx}}(\mathcal{M})\}$. 

\begin{definition}[{\bf Group of isometries of $\mathcal{M}$} ($I\left(\mathcal{M}\right)$)]
\label{theory:def6}
A diffeomorphism $\phi: \mathcal{M} \rightarrow \mathcal{M}$ is an
  isometry if it preserves distance, i.e., $d\left(\phi\left(x\right),
  \phi\left(y\right)\right) = d\left(x, y\right)$. The set
  $I(\mathcal{M})$ of all isometries of $\mathcal{M}$ forms a group
  with respect to function composition.
\end{definition}

Rather than write an isometry
  as a function $\phi$, we will write it as a group action.
  Henceforth, let $G$ denote the group $I(\mathcal{M})$, and for $g
  \in G$, and $x \in \mathcal{M}$, let $g\cdot x$ denote the result of
  applying the isometry $g$ to point $x$. 

  \begin{definition} [Riemannian homogeneous spaces \cite{helgason2001differential}] 
 Given $\mathcal{M}$ and $G$ as defined above, let $G$ acts transitively on $\mathcal{M}$, i.e., given $p, q\in \mathcal{M}$, $\exists g\in G$, such that $q = g\cdot  p$. Let $H = \textsf{Stab}(I)$, where $I$ is the ``origin'' of $\mathcal{M}$ where $\textsf{Stab}(I) = \left\{g \in G| g\cdot I = I\right\}$ is the stabilizer of $I$. Then $\mathcal{M}$ is a Riemannian homogeneous space and can be identified as the quotient space $G/H$.   
  \end{definition} 
 
Some of the examples of Riemannian homogeneous spaces include Euclidean space, hypersphere, hyperbolic space, Lie groups (will be defined next).
\begin{definition}[Lie group \cite{hall2015lie}]
$\mathcal{M}$ is called a Lie group if \begin{inparaenum}[\bfseries (a)] \item $\mathcal{M}$ is a group with the group operation $\circ$ \item the group operations $(g, h)\mapsto g\circ h$ and $g \mapsto g^{-1}$ are smooth. \end{inparaenum}  
\end{definition}  

\begin{definition}[Lie algebra \cite{hall2015lie}]
The tangent space of $\mathcal{M}$ at identity, $I$, i.e., $T_I \mathcal{M}$ is a vector space and is termed as Lie algebra, $\mathfrak{M}$. Lie algebra is a vector space. 
\end{definition}  

Observe the basic properties of a matrix Lie group, $\mathcal{M}$: \begin{inparaenum}[\bfseries (a)] \item the distance on $\mathcal{M}$ can be defined as $d(X, Y) = \|\textsf{logm}\left(X^{-1}Y\right)\|$, here $\textsf{logm}$ is the matrix logarithm and $\|.\|$ is the Frobenius norm \item $\textsf{logm}: \mathcal{M} \rightarrow \mathfrak{m}$ is the mapping from Lie group to Lie algebra and $\textsf{expm}$ is the inverse of this mapping \item Given $X, Y\in \mathcal{M}$, the shortest geodesic from $X$ to $Y$ is given by $\Gamma_X^Y(t) = X\textsf{expm}\left(t\textsf{logm}\left(X^{-1}Y\right)\right)$\end{inparaenum}

In the rest of the paper, we assume $\mathcal{M}$ to be a Riemannian homogeneous space. Moreover, we will assume $\mathcal{M}$ is associated with the Levi-Civita connection: $\nabla: V_{\mathcal{M}} \times V_{\mathcal{M}} \rightarrow V_{\mathcal{M}}$ where $V_{\mathcal{M}}$ is the space of vector fields on $\mathcal{M}$ \cite{boothby1986introduction}. 

Now, we give some definitions including Parallel transport, Fr\'{e}chet mean which are needed in order to define Riemannian normalization. 

\begin{definition}[{\bf Parallel transport on $\mathcal{M}$} ($\Gamma_{p \rightarrow q} \left(\mathbf{v}\right)$)]
\label{theory:def7}
Let $p, q\in \mathcal{M}$ and $\mathbf{v} \in T_p\mathcal{M}$. Let $\gamma: [0,1] \rightarrow \mathcal{M}$ be the (shortest) geodesic with $\gamma(0) = p$ and $\gamma(1) = q$. A vector field $V$ is said to be parallel transport of $\mathbf{v}$ along $\gamma$ provided that $\left\{V(t), t\in [0,1]\right\}$ is a vector field for which $V(0) = \mathbf{v}$. We assign $V(1) \in T_q\mathcal{M}$ to be $\Gamma_{p\rightarrow q} \left(\mathbf{v}\right)$.
\end{definition}

 Note that the term parallel is because of $\nabla_{\gamma'(t)} V(t)|_{t_0} = 0$, for all $t_0\in [0,1]$.

\begin{definition}[{\bf weighted Fr\'{e}chet mean}]
\label{theory:def9}
Given $\left\{X_i\right\}_{i=1}^N \subset \mathcal{M}$, and a set of weights $\left\{w_i\right\}_{i=1}^N \subset (0,1]$ with $\sum_{i=1}^N w_i = 1$ (i.e., $\left\{w_i\right\}$ satisfy convexity constraint), we can define ``the'' weighted Fr\'{e}chet mean (wFM) \cite{frechet1948elements} as the minimizer of the weighted variance, i.e.,
$$
\textsf{wFM}\left(\left\{X_i\right\}, \left\{w_i\right\}\right) = \argmin_{M\in \mathcal{M}} \sum_{i=1}^N w_i d^2(X_i, M).
$$
\end{definition}

We use the following proposition \cite{afsari2011riemannian} to argue that if the samples are within the geodesic ball of aforementioned radius, then the $\textsf{wFM}$ exists and is unique. Note that if $w_i = \sfrac{1}{N}$, for all $i$, then we get ``the'' Fr\'{e}chet mean (FM) defined as
\begin{align}
\textsf{FM}\left(\left\{X_i\right\}\right) = \argmin_{M\in \mathcal{M}} \sum_{i=1}^N d^2(X_i, M).
\label{fmeqn}
\end{align}

Given $\left\{X_i\right\}_{i=1}^N \subset \mathcal{M}$ we will use a provably convergent recursive estimator of $\textsf{wFM}$ as proposed in Chakraborty et al. \cite{chakraborty2018manifoldnet}. The recursive $\textsf{wFM}$ estimator, $M_N$, is defined as
\begin{align}
M_1 = X_1 \qquad\qquad M_{n+1} = \Gamma_{M_n}^{X_{n+1}} \left(\frac{w_{n+1}}{\sum_{j=1}^{n+1}w_j}\right)
\label{fmrec}
\end{align}

Recently in \cite{chakraborty2018manifoldnet}, the authors proposed a manifold valued deep neural network where they defined convolution operator using $\textsf{wFM}$. In the next section, we first formally define Riemannian normalization before recalling the definition of convolution. 

\section{Riemannian normalization}
In this section, we formulate a general normalization scheme on a Riemannian manifold. We propose algorithms for normalization on a general homogeneous space and a Lie group in the subsequent subsections. Before that we formulate the problem of Riemannian normalization in a general form and show that the popular {\it batch normalization}, {\it group normalization}, {\it layer norm} are special cases of our formulation when the manifold is an Euclidean space.

\begin{definition}[{\bf Riemannian normalization}]
\label{theory:def10}
Given $\left\{X_{i_1,i_2,i_3,i_n,i_c}\right\} \subset \mathcal{M}$ with indices $i1, i2, i3$ run over the spatial 3D dimension (correspond to three dimension of a 3D volume), $i_n$ and $i_c$ are the indices over the number of samples and number of channels respectively, the Riemannian normalization normalize the first order and second order moments over specific index (or a set of indices). Let $\mathcal{S}$ be the set over which we desire to perform the normalization. Depending on the construction of the set $\mathcal{S}$ we get different types of normalization. As for an example, if $\mathcal{S} = \left\{X_{i_1,i_2,i_3,i_n,i_c} | i_c = c\right\}$ then it is batch normalization, here $c$ is a channel index $c$. In other words, the batch normalization is over $\left\{(i_1, i_2, i_3, i_n)\right\}$ indices.
\end{definition}

Given a set $\mathcal{S}$, the Riemannian normalization tries to fit a distribution with desired first and second order moments. Before formally defining distribution on a Riemannian homogeneous space, we first give examples of different kinds of Riemannian normalization, i.e., different choices of $\mathcal{S}$.

\begin{enumerate}[\bfseries (a)]
\item {\it Riemannian batch normalization:} If $\mathcal{S} = \left\{X_{i_1,i_2,i_3,i_n,i_c} | i_c = c\right\}$ for a channel $c$, then the normalization is termed as Riemannian batch normalization. Hence,  the batch normalization is over $\left\{(i_1, i_2, i_3, i_n)\right\}$ indices.
\item {\it Riemannian layer normalization:} If $\mathcal{S} = \left\{X_{i_1,i_2,i_3,i_n,i_c} | i_n = n\right\}$ for a sample $n$, then the normalization is termed as Riemannian layer normalization. Hence,  the layer normalization is over $\left\{(i_1, i_2, i_3, i_c)\right\}$ indices.
\item {\it Riemannian instance normalization:} If $\mathcal{S} = \left\{X_{i_1,i_2,i_3,i_n,i_c} | i_c = c, i_n=n \right\}$ for a channel $c$ and sample $n$, then the normalization is termed as Riemannian instance normalization. Hence,  the instance normalization is over $\left\{(i_1, i_2, i_3)\right\}$ indices.
\item {\it Riemannian group ormalization:} If $\mathcal{S} = \left\{X_{i_1,i_2,i_3,i_n,i_c} | i_n=n, i_c \in C_g \right\}$ for a sample $n$ and a channel group $C_g = \{c_1, \cdots, c_n\}$, then the normalization is termed as Riemannian group normalization. Hence,  the group normalization is over $\left\{(i_1, i_2, i_3, i_c)\right\}$ indices but $i_c$ is over a group of channels $C_g$.
\end{enumerate}
A visual description of different kind of Riemannian normalization is shown in Fig. \eqref{fig1}. 
\begin{figure}[!b]
    \setlength{\abovecaptionskip}{-0cm}
    \setlength{\belowcaptionskip}{-0cm} 
        \centering
                \includegraphics[width=0.9\columnwidth]{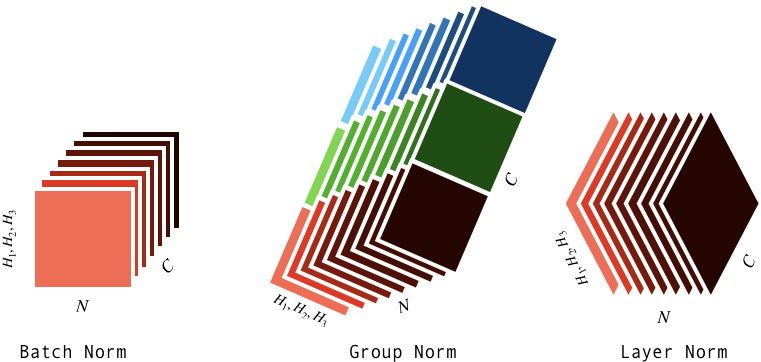}
               \caption{\footnotesize Pictorial description of various kind of Riemannian normalization.}\label{fig1}
\end{figure}

\subsection{Riemannian homogeneous spaces}
In this subsection, we assume $\mathcal{M}$ to be a Riemannian homogeneous space of dimension $m$. We assume the distance $d$ used is induced by the $G$-invariant Riemannian metric where $G$ is the group transitively acts on $\mathcal{M}$. Hence, the isometry group under this distance $d$ is the group $G$. We use $\cdot$ to denote the group action as given by $(g, M) \mapsto g\cdot M$, where, $g\cdot M \in \mathcal{M}$.  

Let $M \in \mathcal{M}$, there exists an isomorphism $\iota: T_M \mathcal{M} \rightarrow \mathbf{R}^m$ from the tangent space at $M$ to the Euclidean space $\mathbf{R}^m$. Now, we are ready to give the definition of Gaussian distribution. 

\begin{definition}[{\bf Gaussian distribution \cite{pennec2004probabilities}}]
\label{theory:def11}
Given a Riemannian homogeneous space $\mathcal{M}$ (of dimension $m$) with the distance $d$ and group $G$ acts of $\mathcal{M}$, we can define Gaussian distribution  with location parameter $M$ and concentration matrix $\Delta$ as:
\begin{align}
    f(X|M;\Delta) &= k\exp\left(-\frac{\mathbf{v}^t\Delta\mathbf{v}}{2}\right)
\end{align}
where the normalization constant $k$ and the covariance matrix $\Sigma$ are given as follows.  
$$
k^{-1} = \int_{\mathcal{M}} \exp\left(-\frac{\mathbf{v}^t\Delta\mathbf{v}}{2}\right) \omega(X)
$$
$$
\Sigma = k \int_{\mathcal{M}} \mathbf{v}\mathbf{v}^t \exp\left(-\frac{\mathbf{v}^t\Delta\mathbf{v}}{2}\right) \omega(X)
$$
here, $\omega:\mathcal{M} \rightarrow [0,\infty)$ is the Riemannian volume density and $\mathbf{v} = \textsf{Exp}^{-1}_M(X)$. This definition amounts to defining truncated Gaussian distribution of the exponential chart map.
\end{definition}

Given a set $\mathcal{S}$ of samples on which we need to apply normalization, in Alg. \eqref{alg1train} which present the training step of Riemannian normalization. 

\begin{algorithm}
{ \KwIn{A batch of samples $\mathcal{S} = \left\{X_i\right\}_{i=1}^N$; bias $g\in G$; running mean $M$; positive diagonal scaling matrix $S \in \mathbf{R}^{m\times m}$.} 
    \KwOut{updated running mean $M$.}
   Compute batch mean, $M_b$ of $\left\{X_i\right\}_{i=1}^N$ using Riemannian metric (incremental FM in Eq. \eqref{fmrec})\;
    Update running mean $M$ by $M_b$ (incremental FM in Eq. \eqref{fmrec})\;
    $X_i\leftarrow \textsf{Exp}\left(\Gamma_{M_b \rightarrow I}\left(\mathbf{v}_i\right)\right)$, where $\mathbf{v}_i =  \textsf{Exp}^{-1}_{M_b}\left(X_i\right)$\;
    $X_i\leftarrow \textsf{Exp}\left(\iota^{-1}\left(S \iota\left(\mathbf{v}_i\right)\right)\right)$, where  $\mathbf{v}_i =  \textsf{Exp}^{-1}_{I}\left(X_i\right)$\;
    $X_i \leftarrow g \cdot X_i$.
       
    \caption{{Training step of normalization on a Riemannian homogeneous space}}
    \label{alg1train}
    }
\end{algorithm}

In Alg. \eqref{alg1test}, we present the testing algorithm. Notice that in training algorithm, we update the running mean of the distribution, while for testing algorithm we use the final learned running mean $M$.

\begin{algorithm}
{ \KwIn{A batch of samples $\mathcal{S} = \left\{X_i\right\}_{i=1}^N$; bias $g\in G$; learned running mean $M$; diagonal scaling matrix $S$.} 
   $X_i\leftarrow \textsf{Exp}\left(\Gamma_{M \rightarrow I}\left(\mathbf{v}_i\right)\right)$, where $\mathbf{v}_i =  \textsf{Exp}^{-1}_{M}\left(X_i\right)$\;
    $X_i\leftarrow \textsf{Exp}\left(\iota^{-1}\left(S \iota\left(\mathbf{v}_i\right)\right)\right)$, where  e $\mathbf{v}_i =  \textsf{Exp}^{-1}_{I}\left(X_i\right)$\;
    $X_i \leftarrow g \cdot X_i$.
       
    \caption{{Testing step of normalization on a Riemannian homogeneous space}}
    \label{alg1test}
    }
\end{algorithm}

Note that, the parallel translate operation does not guarantee the FM of the samples in $\mathcal{S}$, as given the mean $M$ and the desired mean $I$, although there exists a group element $g_M \in G$ such that $g_M\cdot I = M$ (as a property of the Riemannian homogeneous space), $g_M$ does not have a closed form in general. Hence, we will focus on a subclass of Riemannian homogeneous spaces, namely Lie groups, where because of the group inverses we can get a closed form of $g_M$. 

Note that, the above algorithms can be applicable to a general Riemannian manifold $\mathcal{M}$ with closed form for geodesic, parallel transport and we will use $G = I(\mathcal{M})$, the isometry group. Before giving the formulation for Lie groups, we present two examples of homogeneous spaces with the appropriate operations needed for Riemannian normalization.

\subsubsection{Riemannian homogeneous spaces: some examples}

\paragraph{SPD}: Let $\mathcal{M}$ be the manifold of $n\times n$ symmetric positive definite matrices with affine-invariant metric. Below, we give closed form of the operations needed in the normalization algorithm.

\begin{enumerate}[\bfseries (a)]
\item {\it Distance:} $d(X,Y) = \|\textsf{logm}\left(X^{-1}Y\right)\|$.
\item {\it $G$:} The group that acts on $\mathcal{M}$ (isometry group) is $G = \textsf{GL}(n)$, $n\times n$ invertible matrices. 
\item {\it Group action:} $g\cdot X \mapsto gXg^T$.
\item {\it Log map:} $\textsf{Exp}^{-1}_X(Y) = X^{1/2} \textsf{logm}\left(X^{-1/2}YX^{-1/2}\right)X^{1/2}$, where, $\textsf{logm}(Z) = U\log(D)U^T$, where $Z=UDU^T$.
\item {\it Exp map:} $\textsf{Exp}_X(V) = X^{1/2} \textsf{expm}\left(X^{-1/2}VX^{-1/2}\right)X^{1/2}$, where, $\textsf{expm}(Z) = U\exp(D)U^T$, where $Z=UDU^T$.
\item {\it Parallel transport:} $\Gamma_{X\rightarrow Y}(V) = Y^{1/2}\left(X^{-1/2}VX^{-1/2}\right)Y^{1/2}$.
\end{enumerate}

\paragraph{$\mathbf{S}^{n}$}: Let $\mathcal{M}$ be $n$-dimensional unit hypersphere with arc-length metric. Below, we give closed form of the operations needed in the normalization algorithm.

\begin{enumerate}[\bfseries (a)]
\item {\it Distance:} $d(\mathbf{x},\mathbf{y}) = \arccos\left(\mathbf{x}^t\mathbf{y}\right)$.
\item {\it $G$:} The group that acts on $\mathcal{M}$ (isometry group) is $G = \textsf{SO}(n)$, $n\times n$ special orthogonal matrices. 
\item {\it Group action:} $g\cdot \mathbf{x} \mapsto g\mathbf{x}$.
\item {\it Log map:} $\textsf{Exp}^{-1}_{\mathbf{x}}(\mathbf{y}) = \frac{\sin(\theta)}{\theta \left(\mathbf{y} - \mathbf{x}\cos(\theta)\right)}$, where $\theta = d(\mathbf{x}, \mathbf{y})$. 
\item {\it Exp map:} $\textsf{Exp}_{\mathbf{x}}(\mathbf{v}) = \cos(\|\mathbf{v}\|)\mathbf{x} + \sin(\|\mathbf{v}\|)\frac{\mathbf{v}}{\|\mathbf{v}\|}$. 
\item {\it Parallel transport:} $\Gamma_{\mathbf{x}\rightarrow \mathbf{y}}(\mathbf{v}) = \left(\mathbf{v} - \mathbf{w}\frac{\mathbf{w}^t\mathbf{v}}{\|\mathbf{w}\|^2}\right) + \frac{\mathbf{w}^t\mathbf{v}}{\|\mathbf{w}\|^2} \left(\mathbf{x}\left(-\sin(\|\mathbf{w}\|)\|\mathbf{w}\|\right) + \mathbf{w}\cos(\|\mathbf{w}\|)\right)$, where $\mathbf{w} = \textsf{Exp}^{-1}_{\mathbf{x}}\left(\mathbf{y}\right)$.
\end{enumerate}

\subsection{Matrix Lie groups}
In this subsection, we assume $\mathcal{M}$ to be matrix Lie group. As mentioned before, given $X, Y \in \mathcal{M}$, we define the distance as $d(X,Y) = \|\textsf{logm}\left(X^{-1}Y\right)\|$. Notice that, this metric invariant to the left group operation. Formally, given $Z \in \mathcal{M}, d(ZX, ZY) = d(X, Y)$. Hence, the isometry group, $G$ is given by $G =\mathcal{M}$ with respect to the left group operation.

Before defining normalization for matrix Lie groups, we first define Gaussian distribution for matrix Lie groups. Note that although the earlier definition of Gaussian distribution on a Riemannian homogeneous space can be applied here, here we gave a different definition of Gaussian distribution which can be used to define computationally more efficient Riemannian normalization for special cases of matrix Lie groups. 

\begin{definition}[{\bf Gaussian distribution \cite{chakraborty2019statistics}}]
\label{theory:def8}
Given a Riemannian homogeneous space $\mathcal{M}$ with the distance $d$, induced from a Riemannian metric (and group $G$ acts of $\mathcal{M}$), we can define Gaussian distribution  with location parameter $M\in \mathcal{M}$ and variance $\sigma^2>0$, denoted by $\mathcal{N}\left(M, \sigma^2\right)$ as:
\begin{align}
    f(X|M,\sigma^2) &= k(\sigma)\exp\left(-\frac{d(X,M)^2}{2\sigma^2}\right)
    \label{gauss}
\end{align}
where, $k$ is the normalizing constant.
\end{definition}

Before presenting the algorithm of Remannian normalization on matrix Lie groups, we first start with stating some propositions. 

\begin{proposition}
Given $\left\{X_i\right\}\subset \mathcal{M}$ i.i.d. samples drawn from $\mathcal{N}\left(M, \sigma^2\right)$, the maximum likelihood estimator (MLE) of $M$ is the sample Fr\'{e}chet mean (FM) of $\left\{X_i\right\}$.
\label{mleprop}
\end{proposition}
\begin{proof}
From Eq. \eqref{gauss}, we can get the log-likelihood, $\ell\left(M; \left\{X_i\right\}, \sigma^2\right)$ as
\begin{align*}
\ell\left(M; \left\{X_i\right\}, \sigma^2\right) = \log(k(\sigma)) - \sum_{i=1}^N \frac{d(X_i, M)^2}{2\sigma^2}
\end{align*}

Now, maximizing $\ell\left(M; \left\{X_i\right\}, \sigma^2\right)$ is equivalent to minimizing $\sum_{i=1}^N d(X_i, M)^2$. Hence, using Eq. \eqref{fmeqn}, we can conclude that the MLE of $M$ is the FM of $\left\{X_i\right\}$.
\end{proof}
Now, using the MLE and FM equivalence as showed in Prop. \eqref{mleprop}, we can state the following proposition. 

\begin{proposition}
Given $X \sim \mathcal{N}\left(M, \sigma^2\right)$ with parameters $M \in \mathcal{M}$ and $\sigma^2>0$, $ZX \sim \mathcal{N}\left(ZM, \sigma^2\right)$, for all $Z \in \mathcal{M}$.
\label{prop2}
\end{proposition}
\begin{proof}
In order to prove the proposition it is sufficient to show that $f(X|M,\sigma^2) = c f(ZX|ZM,\sigma^2)$  using Eq. \eqref{gauss} in Def. \eqref{theory:def8} for some constant $c>0$. Observe that,
\begin{align*}
f(ZX|ZM,\sigma^2) &= k(\sigma)\exp\left(-\frac{d(ZX,ZM)^2}{2\sigma^2}\right) \\
&= k(\sigma) \exp\left(-\frac{\|\textsf{logm}\left(X^{-1}Z^{-1}ZM\right)\|^2}{2\sigma^2}\right) \\
&= f(X|M,\sigma^2)
\end{align*}
\end{proof}

\begin{proposition}
Given $X \sim \mathcal{N}\left(I, \sigma^2\right)$ with parameters $I \in \mathcal{M}$ ($I$ to be the identity element) and $\sigma^2>0$, then $Y := \textsf{expm}\left(s\textsf{logm}\left(X\right)\right) \sim \mathcal{N}\left(I, s^2\sigma^2\right)$, for all $Z \in \mathcal{M}$ for all $s>0$.
\label{prop3}
\end{proposition}

\begin{proof}
Similar to before, it is sufficient to show that $f(X|I,\sigma^2) = c f(Y|I,s^2\sigma^2)$ where, $Y = \textsf{expm}\left(s\textsf{logm}\left(X\right)\right)$, for some constant $c$. Observe that,
\begin{align*}
f(Y|I,s^2\sigma^2) &= k(s\sigma)\exp\left(-\frac{d(Y,I)^2}{2s^2\sigma^2}\right) \\
&= k(s\sigma) \exp\left(-\frac{\|\textsf{logm}\left(Y\right)\|^2}{2s^2\sigma^2}\right) \\
&= k(s\sigma) \exp\left(-\frac{\|\textsf{logm}\left(X\right)\|^2}{2\sigma^2}\right) \\
&= c f(X|I,\sigma^2)
\end{align*}
for some $c>0$. 
\end{proof}

As a corollary of Prop. \eqref{prop3}, we can state the following.
\begin{proposition}
Given $\left\{X_i\right\}_{i=1}^N \subset \mathcal{M}$ and $\left\{w_i\right\}$ satisfying convexity constraint, let $I = \textsf{wFM}\left( \left\{X_i\right\}, \left\{w_i\right\}\right)$ be the wFM. Then, for all $s>0$, $I = \textsf{wFM}\left( \left\{Y_i\right\}, \left\{w_i\right\}\right)$ be the wFM of $\left\{Y_i\right\}$, where, $Y_i\textsf{expm}\left(s\textsf{logm}\left(X_i\right)\right)$. 
\label{prop4}
\end{proposition}
\begin{proof}
\begin{align*}
\sum_{i=1}^N w_i d^2(Y_i, I) &=  s^2\sum_{i=1}^N w_i d^2(X_i, I) 
\end{align*}
Hence, $I$ to be wFM of $\left\{X_i\right\}$ if and only if $I$ to be wFM of $\left\{Y_i\right\}$.
\end{proof}
Using Props. \eqref{prop3} and \eqref{prop4} we can propose our formulation of normalization in Alg. \eqref{alg2train}. Before giving the testing algorithm, we like to discuss some key points of Alg. \eqref{alg2train} as listed next. \begin{inparaenum}[\bfseries (a)] \item In line 3, we adjust the batch mean to be $I$ (the identity element of $\mathcal{M}$) \item In line 4, the batch mean remains unchanged (courtesy of Prop. \ref{prop4}) but the scaling parameter of the distribution is scaled as stated in Prop. \ref{prop3} \item In line 5, we change the batch mean from $I$ to $g$. \end{inparaenum}

\begin{algorithm}
{ \KwIn{A batch of samples $\mathcal{S} = \left\{X_i\right\}_{i=1}^N$; bias $g\in \mathcal{M}$ (here $G = \mathcal{M}$ as Lie group); running mean $M$; scale factor $s>0$.} 
    \KwOut{updated running mean $M$.}
   Compute batch mean, $M_b$ of $\left\{X_i\right\}_{i=1}^N$ using Riemannian metric (incremental FM in Eq. \eqref{fmrec})\;
    Update running mean $M$ by $M_b$ (incremental FM in Eq. \eqref{fmrec})\;
    $X_i\leftarrow M_b^{-1} X_i$\;
    $X_i \leftarrow \textsf{expm}\left(s\textsf{logm}\left(X_i\right)\right)$\;
    $X_i \leftarrow g X_i$.
       
    \caption{{Training step of normalization on a Lie group}}
    \label{alg2train}
    }
\end{algorithm}

The testing algorithm is similar to before and is presented next in Alg. \eqref{alg2test}. The testing algorithm inputs the updated running mean $M$ and bias $g$ and scale $s$ from training algorithm in Alg. \eqref{alg2train} and changes the distribution of the test samples accordingly. 

\begin{algorithm}
{ \KwIn{A batch of samples $\mathcal{S} = \left\{X_i\right\}_{i=1}^N$; bias $g\in \mathcal{M}$; learned running mean $M$; scale factor $s>0$.} 
    $X_i\leftarrow M^{-1} X_i$\;
    $X_i \leftarrow \textsf{expm}\left(s\textsf{logm}\left(X_i\right)\right)$\;
    $X_i \leftarrow g X_i$.
       
    \caption{{Testing step of normalization on a Lie group}}
    \label{alg2test}
    }
\end{algorithm}

Now, we present some examples of matrix Lie groups.

\subsubsection{Lie groups: some examples}

\paragraph{SPD}: Let $\mathcal{M}$ be the manifold of $n\times n$ symmetric positive definite matrices with log-Euclidean metric. As shown in \cite{arsigny2006log}, $\textsf{SPD}$ with log-Euclidean metric is a matrix Lie group. Below, we give closed form of the operations needed in the normalization algorithm.

\begin{enumerate}[\bfseries (a)]
\item {\it Distance:} $d(X,Y) = \|\textsf{logm}(X) - \textsf{logm}(Y)\|$.
\item {\it $G$:} The group that acts on $\mathcal{M}$ (isometry group) is $G = \textsf{SO}(n)$, $n\times n$ special orthogonal matrices. 
\item {\it Group action:} $g\cdot X \mapsto gXg^T$.
\item {\it Log map:} $\textsf{Exp}^{-1}_X(Y) = \textsf{logm}(Y) - \textsf{logm}(X)$.
\item {\it Exp map:} $\textsf{Exp}_X(V) = \textsf{expm}\left(V + \textsf{logm}(X) \right)$.
\end{enumerate}

\paragraph{SO}: Let $\mathcal{M}$ be the manifold of $n\times n$ special orthogonal matrices. Below, we give closed form of the operations needed in the normalization algorithm.

\begin{enumerate}[\bfseries (a)]
\item {\it Distance:} $d(X,Y) = \|\textsf{logm}(X^TY)\|$.
\item {\it $G$:} The group that acts on $\mathcal{M}$ (isometry group) is $G = \textsf{SO}(n)$. 
\item {\it Group action:} $g\cdot X \mapsto gX$.
\item {\it Log map:} $\textsf{Exp}^{-1}_X(Y) = X\textsf{logm}(X^TY)$.
\item {\it Exp map:} $\textsf{Exp}_X(V) = X\textsf{expm}\left(X^TV\right)$.
\item {\it Parallel transport:} $\Gamma_{X\rightarrow Y}(V) = YX^TV$.
\end{enumerate}

\subsection{Architecture for network with manifold normalization}

In this section, we present the basic building blocks for manifold valued deep learning: \begin{inparaenum}[\bfseries (a)] \item the {\it ManifoldConv} layer as proposed in \cite{chakraborty2018manifoldnet} using wFM operator proposed in Eq. \eqref{fmrec} \item tangent ReLU ({\it tReLU}) as non-itineraries \item manifold normalization block \item manifold valued fully connected ({\it ManifoldFC}) layer  \end{inparaenum}. For completeness, we present the definition of manifoldconv, tReLU and manifoldfc blocks here.

{\it ManifoldConv:} Given $\{X_i\|_{i=1}^N$ and weights (to be learned) $\{w_i\}_{i=1}^N$ satisfying convexity constraint, the output of ManifoldConv is $\textsf{wFM}\left(\{X_i\}, \{w_i\}\right)$ as defined in Eq. \eqref{fmrec}.

{\it tReLU:} This layer takes $X \in \mathcal{M}$ as input and returns $\textsf{Exp}_I\left(\text{ReLU}\left(\textsf{Exp}^{-1}_I\left(X\right)\right)\right)$. 

{\it ManifoldFC:} Given $\{X_i\|_{i=1}^N$ as input, ManifoldFC returns $\{d(X_i, M)\}_{i=1}^N \subset \mathbf{R}$, where $M = \textsf{FM}\left(\{X_i\}\right)$. 

A sample manifoldnet architecture with normalization is presented in Fig. \eqref{manifoldnet}.
\begin{figure}[!ht]
    \setlength{\abovecaptionskip}{-0cm}
    \setlength{\belowcaptionskip}{-0cm} 
        \centering
               \includegraphics[width=1\columnwidth]{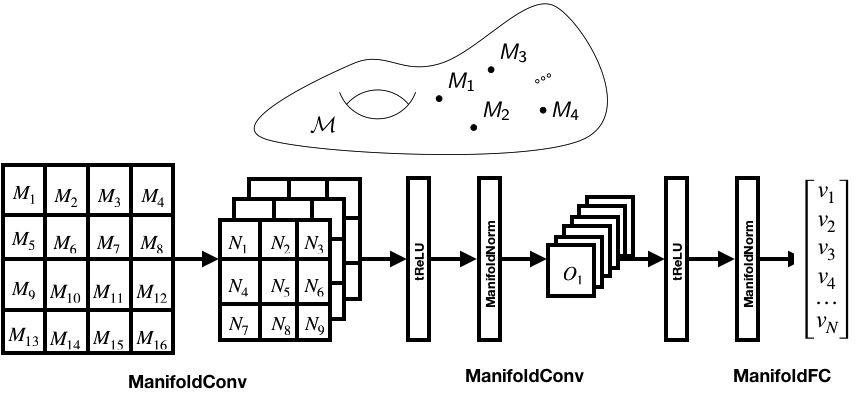}
               \caption{\footnotesize{A sample manifoldnet architecture }}\label{manifoldnet}
\end{figure}
Now, we present some proof of concepts type experiments showing the effectiveness of our proposed normalization technique for classification task.

\section{Experiments}\label{results}

In this section, we perform two sets of experiments. We perform a proof of concept type synthetic experiment on moving MNIST dataset. Then, we perform analysis on brain image classification on Human Connectome Project (HCP)  dataset.

\subsection{Moving MNIST: Moving pattern classification}

We
generated the Moving MNIST data according to the algorithm proposed in \cite{srivastava2015unsupervised}. In this experiment, we classify the moving patterns of different digits. For each moving pattern, we generated $1000$ sequences with length $20$ showing $2$ digits 
moving in the same pattern in a $64\times 64$ frame. The moving speed and the direction 
are fixed inside each class, but the digits are chosen randomly. 
In this experiment, the difference in the moving angle from two sequences across different classes is at least $5^\circ$. 

\setlength{\intextsep}{2pt}%
\setlength{\columnsep}{5pt}%
\begin{table*}[!ht]
\vspace*{-0.0cm}
\setlength{\abovecaptionskip}{-0cm}
\setlength{\belowcaptionskip}{-0.0cm} 
   \centering
   \scalebox{0.9}{
\begin{tabular}{cccccc} 
\topline\myrowcolour
 & & {\bf time (s)} & \multicolumn{3}{c}{{\bf Test acc.}} \\
\arrayrulecolor\hhline{---~~~}\arrayrulecolor{rulecolor}\hhline{~~~---}\myrowcolour
\multirow{-2}{*}{\bf Model} & \multirow{-2}{*}{\bf \# params.} & {\bf / epoch} & $30^\circ$vs. $60^\circ$ & $10^\circ$vs. $15^\circ$ & $10^\circ$vs. $15^\circ$vs. $20^\circ$ \\
\midtopline
ManifoldNet-GN & ${1072}$ & $\sim 2.7$ &  $\highest{1.00\pm 0.00}$ & $\highest{1.00 \pm 0.00}$ & $\highest{1.00 \pm 0.00}$ \\
ManifoldNet-BN & ${1052}$ & $\sim 2.7$ &  $\highest{1.00\pm 0.00}$ & $\highest{1.00 \pm 0.00}$ & ${0.98 \pm 0.01}$ \\
ManifoldNetLE-BN & ${752}$ & $\sim 2.7$ &  $\highest{1.00\pm 0.00}$ & ${0.99 \pm 0.01}$ & ${0.98 \pm 0.02}$ \\
ManifoldNet & $\highest{738}$ & $\sim 2.7$ &  $\highest{1.00\pm 0.00}$ & ${0.99 \pm 0.01}$ & ${0.97 \pm 0.02}$ \\
MVC-Net & 13564 & $\sim 4.1$ &  $\highest{1.00\pm 0.00}$ & ${0.99 \pm 0.01}$ & ${0.98 \pm 0.01}$ \\
ManifoldDCNN & ${1517}$ & $\sim 4.3$ & $\highest{1.00\pm 0.00}$ & $\highest{1.00 \pm 0.00}$ & $ {0.95 \pm 0.01}$ \\
SPD-SRU & 1559 & $\sim 6.2$ & $\highest{1.00\pm 0.00}$ & ${0.96 \pm 0.02}$ & ${0.94 \pm 0.02}$ \\
TT-GRU & $2240$ & $\sim\highest{2.0}$ & $\highest{1.00 \pm 0.00}$ & $0.52 \pm 0.04$ & $0.47 \pm 0.03$ \\
TT-LSTM & $2304$ & $\sim\highest{2.0}$ & $\highest{1.00 \pm 0.00}$ & $0.51 \pm 0.04$ & $0.37 \pm 0.02$ \\
SRU & $159862$ & $\sim 3.5$ & $\highest{1.00 \pm 0.00}$ & $0.75 \pm 0.19$ & $0.73 \pm 0.14$ \\
LSTM & $252342$ & $\sim 4.5$ & $0.97 \pm 0.01$ & $0.71 \pm 0.07$ & $0.57 \pm 0.13$ \\
\bottomline
\end{tabular}
}
\caption{\footnotesize Comparative results on Moving MNIST.}
\label{results:tab1}
\end{table*}

In Table \eqref{results:tab1}, the results show that our method achieves 
the best test accuracy. We compared with several baselines including ManifoldNet \cite{chakraborty2018manifoldnet}, MVC-Net \cite{bouza2020mvc}, ManifoldDCNN \cite{zhen2019dilated}, SPD-SRU \cite{chakraborty2018statistical}, Tensor train (TT)-GRU, -LSTM \cite{yang2017tensor} and Euclidean GRU and LSTM \cite{chung2014empirical,hochreiter1997long}. We have used three variants of manifold normalization, including batch norm (BN), group norm (GN) with number of channels to be $4$ in a group, batch norm with log-Euclidean metric (LE-BN) with Lie group representation. We have used three layers of manifold convolution with a standard CNN in the beginning. The kernel of standard CNN we use has size $5\times 5$ with the input channel and
output channel set to $5$ and $10$ respectively. We have used $\textsf{tReLU}$ as non-linearities in between manifold convolution layers. All parameters are chosen in a way to use the
fewest number of parameters without deteriorating the test accuracy. We can see that using manifold  normalization the test performance increases with a small increase in the number of parameters and without sacrificing inference time.

\subsection{Real dataset}
The dataset for our method is a subset of the Human Connectome Project (HCP) \cite{van2013wu}. 
We randomly sampled $450$ subjects from the whole dataset which have the preprocessed 3T diffusion-weighted MR images (dMRI). 
The detail of the demographics is shown in Table \eqref{stat_connectome}. 
All the raw dMRI images are preprocessed with the HCP diffusion pipeline with FSL's `eddy' \cite{andersson2016integrated}.
After the correction, the orientation distribution functions (ODF) is generated using the Diffusion Imaging in Python (DIPY) toolbox\cite{garyfallidis2014dipy}. The dimension of  ODF representation is $361$ (lies on $\mathbf{S}^{361}$). We chose a region of interest (ROI) from the center of the 3D volume of the brain of the size $32\times 32\times 32$. 
A sample ROI with functional anisotropy (FA) map and the corresponding ODF are shown in FIg. \eqref{odf_sample}.
\begin{figure}[!ht]
    \setlength{\abovecaptionskip}{-0cm}
    \setlength{\belowcaptionskip}{-0cm} 
        \centering
               \includegraphics[width=1\columnwidth]{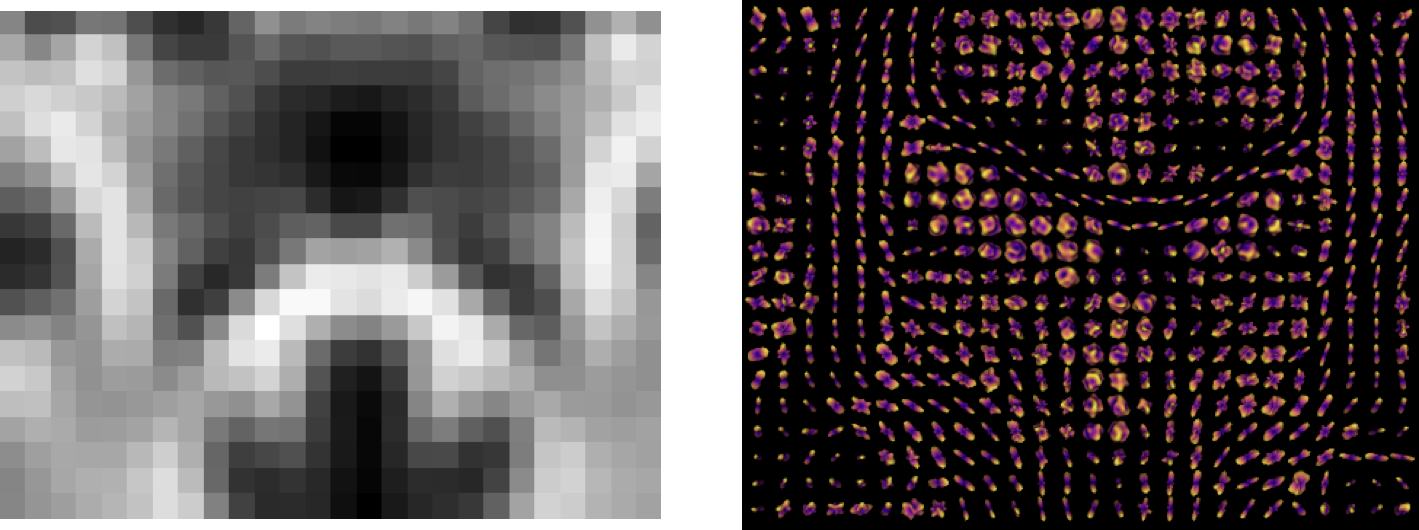}
               \caption{\footnotesize ({\it Left:}) Functional anisotropy (FA) map in a chosen ROI, ({\it Right:}) corresponding ODF }\label{odf_sample}
\end{figure}

We have performed classification of male versus female with the input as the ODF representation and the result is reported in Table \eqref{results:tab2}. We have performed random $90\%$ training and the rest for testing and report the average over $10$ independent runs. We have compared with ResNet34 \cite{he2016deep} as  baseline. We see that with Riemannian group normalization with $4$ channels per group, we can achieve the maximum testing accuracy with a very few number of parameters. 

\setlength{\intextsep}{2pt}%
\setlength{\columnsep}{5pt}%
\begin{table}[!ht]
\vspace*{-0.0cm}
\setlength{\abovecaptionskip}{-0cm}
\setlength{\belowcaptionskip}{-0.0cm} 
   \centering
   \scalebox{0.95}{
\begin{tabular}{cccc} 
\topline
 & & {\bf time (s)} & {{\bf Test acc.}} \\
{\bf Model} & {\bf \# params.} & {\bf / sample} &  \\
\midtopline
ManifoldNet-GN & ${160K}$ & $\sim 0.03$ &  $\highest{0.98\pm 0.01}$ \\
ManifoldNet-BN & ${158K}$ & $\sim 0.03$ &  ${0.96\pm 0.02}$ \\
ManifoldNet & $\highest{100K}$ & ${\sim 0.02}$ &  ${0.93\pm 0.03}$ \\
ResNet34 & ${30M}$ & $\highest{\sim 0.009}$ &  ${0.78\pm 0.05}$ \\
\bottomline
\end{tabular}
}
\caption{\footnotesize Comparative results on HCP dataset for classification of male vs. female.}
\label{results:tab2}
\end{table}

\begin{table}[!ht]
\vspace*{-0.0cm}
\setlength{\abovecaptionskip}{-0cm}
\setlength{\belowcaptionskip}{-0cm} 
   \centering
   \scalebox{0.90}{
\begin{tabular}{cccc|cc} 
\topline\myrowcolour
 \multicolumn{4}{c|}{{\bf Age}} & \multicolumn{2}{c}{{\bf Gender}} \\
 \myrowcolour
 22-25 & 26-30 & 31-35 & 36+ & Female & Male\\
 \midtopline
 99 & 194 & 150 & 7 & 228(50.7\%) & 222(49.3\%)\\
\bottomline
\end{tabular}
}
\caption{\footnotesize The demographics used in the study.}
\label{stat_connectome}
\end{table}

\section{Conclusions}
Non-Euclidean data and manifold valued data have so far gained some attention in research community. Most of these developments are based on applications to analyze medical valued images using CNNs or RNNs. But similar to standard convolutional neural network, these manifold valued methods often have difficulties including gradient explosion specifically for deeper networks. In this work, we proposed Riemannian normalization techniques for manifold valued data. Analogous to the standard CNN, we showed that using our proposed Riemannian normalization we can get the desired first and second order moments. Furthermore, we have shown that our proposed normalization technique can achieve better classification accuracies for both synthetic and real datasets including publicly available medical imaging human connectome project (HCP) dataset.

\section*{Acknowledgments}
We thank Xingjian Zhen at UW Madison for processing the HCP dataset.

{\small
\bibliographystyle{ieee_fullname}
\bibliography{citations}
}

\end{document}